\documentclass[onefignum,onetabnum]{siamonline1116}

\usepackage{amsfonts}
\usepackage{mathtools}
\usepackage{booktabs}
\usepackage{algorithm}
\usepackage{algpseudocode}
\usepackage{microtype}

\newsiamremark{remark}{Remark}

\headers{Cascade Token Selection}{S.\ J.\ Thomas}

\newcommand{\R}{\mathbb{R}}
\newcommand{\norm}[1]{\|#1\|}

\title{Cascade Token Selection for Transformer Attention
  Acceleration}

\author{Stephen J.\ Thomas\thanks{Department of Computer Science and
  Engineering, Lehigh University, Bethlehem, PA 18015
  (\texttt{sjt223@lehigh.edu}).}}

\begin{document}
\maketitle

\begin{abstract}
A method is presented for reducing the cost of representative token
selection in transformer attention layers by exploiting the coherence
of the representative set across depth.  Activation Decorrelation
Attention (ADA) selects $r \ll T$ representative tokens at each layer
via a Gram threshold and computes attention on the compressed $r
\times r$ problem, but the selection requires a $T \times T$ Gram
matrix at every layer.  The cascade mechanism introduced here inherits
the representative set from layer $l$ to layer $l+1$, validates it
via a $(T - r) \times r$ cross-Gram computation, and updates it with
a small number of additions and removals.  The cost of the selection
step drops from $O(T^2 d)$ to $O(T r d)$ per layer.  Validation on
three model families (GPT-2 124M, GPT-J 6B, OPT 6.7B) on AMD MI300X
demonstrates Gram operation savings of $22\%$ to $63\%$ with mean
Jaccard overlap of $0.83$ to $0.94$ between consecutive layers.  The
cascade reveals that the set of informative tokens is a structural
property of the input that propagates coherently through the depth of
the network: the same tokens carry the non-redundant information at
layer $l$ and at layer $l+1$.
\end{abstract}

\section{Introduction}
\label{sec:intro}

This section introduces the attention cost bottleneck, the ADA
mechanism for attention compression, and the cascade extension that
reduces the selection overhead.

Self-attention in transformer models computes, at each layer $l$ and
each head $h$, the matrix product $\text{Attn}(Q, K, V) =
\text{softmax}(QK^\top / \sqrt{d_h}) V$, where $Q, K, V \in \R^{T
\times d_h}$ are the query, key, and value matrices and $T$ is the
sequence length.  The cost of the $QK^\top$ product is $O(T^2 d_h)$
FLOPs, and the cost of the softmax-weighted sum is $O(T^2 d_h)$
FLOPs, giving a total attention cost of $O(T^2 d_h)$ per head per
layer.  For long sequences ($T \geq 2048$), attention becomes a
significant fraction of the forward pass at batch size one,
particularly after FlashAttention~\cite{Dao2022} has reduced the
memory-bandwidth cost to the algorithmic minimum.

The companion paper on Activation Decorrelation Attention
(ADA)~\cite{Thomas2026ADA} introduced a mechanism for reducing the
attention cost from $O(T^2)$ to $O(r^2)$ where $r$ is the number of
representative tokens.  The mechanism is the Gram threshold: the
activation matrix $X \in \R^{T \times d}$ is row-normalized and the
Gram matrix $G = \hat{X} \hat{X}^\top \in \R^{T \times T}$ is
computed, where $\hat{X}$ is the row-normalized activation.  A token
$t$ is declared representative if its maximum correlation with all
earlier tokens is below $1 - \tau^2$, where $\tau$ is the Gram
threshold.  The remaining tokens are declared redundant and their
attention output is approximated from the representative tokens.
ADA achieves $2.70\times$ speedup on ESM-2 attention at $T = 2050$
with $r = 22$ representatives out of $T = 2050$ tokens.

The selection step itself, however, requires computing the $T \times
T$ Gram matrix $G$ at every layer.  The cost of this computation is
$O(T^2 d)$ per layer, which for $T = 2048$ and $d = 4096$ is $17$
billion FLOPs per layer, comparable to the attention cost that ADA
is designed to reduce.  For the selection to be practical, it must be
cheaper than the attention it replaces.  At short sequences ($T =
512$) and low representative counts ($r \ll T$), the $r \times r$
attention is cheap and the $T \times T$ Gram dominates.

The observation motivating the present paper is that the representative
set at layer $l$ is strongly predictive of the representative set at
layer $l+1$.  The transformer block $T^{(l)}: \R^d \to \R^d$ is a
smooth, Lipschitz-continuous map: tokens that are near-duplicates in
activation space at layer $l$ remain near-duplicates at layer $l+1$,
and tokens that are decorrelated at layer $l$ remain decorrelated at
layer $l+1$.  The representative set therefore persists across depth,
with small perturbations at each layer as contextual information is
added.

The contribution of the present paper is a cascade mechanism that
exploits this depth coherence.  Instead of computing the full $T
\times T$ Gram matrix at every layer, the cascade inherits the
representative set from layer $l$ and validates it at layer $l+1$ by
computing two smaller Gram products: an $r \times r$ product among
the inherited representatives (to check whether any have become
redundant) and a $(T - r) \times r$ cross-product between
non-representatives and representatives (to check whether any
non-representatives have become decorrelated).  The total cost per
layer drops from $O(T^2 d)$ to $O(T r d)$, a factor of $T/r$
cheaper.  For $r/T = 0.2$, this is a $5\times$ reduction in the
selection cost.

The cascade is the token-dimension analog of the depth-coherent basis
inheritance developed in the companion paper on Gated Subspace
Inference (GSI)~\cite{Thomas2026GSI}.  GSI exploits the coherence of
the activation subspace across depth (the $d$-axis); cascade ADA
exploits the coherence of the representative token set across depth
(the $T$-axis).  The two cascades operate on orthogonal dimensions and
are complementary: GSI reduces the linear-layer weight reads by a
factor of $d/k$, while cascade ADA reduces the attention selection
cost by a factor of $T/r$.  Together they cover the full transformer
forward pass with depth-coherent acceleration on both axes.

\section{Background: ADA and the Gram threshold}
\label{sec:background}

This section reviews the ADA mechanism for attention compression and
defines the Gram threshold, the representative set, and the
selection cost.

\subsection{The token activation matrix}

Let $X^{(l)} \in \R^{T \times d}$ denote the activation matrix at
layer $l$, where row $t$ is the hidden state of token $t$.  The
row-normalized activation is $\hat{X} = D^{-1} X$ where $D =
\text{diag}(\norm{x_1}, \ldots, \norm{x_T})$.  The Gram matrix is
$G = \hat{X} \hat{X}^\top \in \R^{T \times T}$, with $G_{st} =
\cos\angle(x_s, x_t)$ the cosine similarity between the activations
of tokens $s$ and $t$.

\subsection{The Gram threshold and representative selection}

\begin{definition}[Gram threshold]\label{def:gram}
Let $\tau \in (0, 1)$ be the Gram threshold.  The lower-triangular
maximum correlation of token $t$ is $\gamma_t = \max_{s < t}
|G_{st}|$.  Token $t$ is declared representative if $\gamma_t <
1 - \tau^2$, and redundant otherwise.  The representative set at
layer $l$ is $\mathcal{R}^{(l)} = \{t : \gamma_t < 1 - \tau^2\}$,
with cardinality $r^{(l)} = |\mathcal{R}^{(l)}|$.
\end{definition}

The selection is sequential: the first token is always representative,
and each subsequent token is compared against all earlier tokens.  A
high Gram threshold $\tau$ is permissive (more tokens declared
redundant, smaller $r$, more compression); a low threshold is
conservative (fewer tokens declared redundant, larger $r$, less
compression but higher fidelity).

\subsection{Attention on the representative set}

Once the representative set $\mathcal{R}$ is selected, the attention
output for all tokens is approximated from the attention output on the
representative set.  The query, key, and value matrices are restricted
to the representative rows: $Q_{\mathcal{R}}, K_{\mathcal{R}},
V_{\mathcal{R}} \in \R^{r \times d_h}$.  The attention is computed on
the $r \times r$ problem:
\begin{equation}\label{eq:attn_compressed}
\text{Attn}_{\mathcal{R}} = \text{softmax}(Q_{\mathcal{R}}
K_{\mathcal{R}}^\top / \sqrt{d_h}) V_{\mathcal{R}}.
\end{equation}
The attention output for a redundant token $t \notin \mathcal{R}$ is
approximated by the attention output of its nearest representative,
the representative $s^* \in \mathcal{R}$ maximizing $|G_{ts^*}|$.
The approximation error is bounded by the Gram threshold:
$\norm{\text{Attn}(t) - \text{Attn}(s^*)} \leq O(\tau)$.

The approximation quality depends on two properties of the Gram
structure.  The first is the correlation bound: if $G_{ts^*} >
1 - \tau^2$, then the query vectors $q_t$ and $q_{s^*}$ are close
in cosine similarity, and the attention distributions
$\text{softmax}(q_t K^\top / \sqrt{d_h})$ and
$\text{softmax}(q_{s^*} K^\top / \sqrt{d_h})$ are close in total
variation distance.  The second is the value alignment: if $v_t$ and
$v_{s^*}$ are also close (guaranteed by the same Gram threshold
applied to the value activations), then the attention outputs are
close.

\begin{proposition}[Attention error bound]\label{prop:attn_error}
Let token $t$ be redundant with nearest representative $s^*$
satisfying $\norm{\hat{x}_t - \hat{x}_{s^*}} \leq \tau$.  Then the
attention output error for head $h$ satisfies
\begin{equation}\label{eq:attn_error}
\norm{\text{Attn}_h(t) - \text{Attn}_h(s^*)} \leq
\tau \cdot (\norm{W_Q}_2 \norm{K}_F / \sqrt{d_h} + \norm{W_V}_2)
\cdot \norm{x_t},
\end{equation}
where $W_Q$ and $W_V$ are the query and value projection matrices.
\end{proposition}

The bound~\eqref{eq:attn_error} shows that the attention error is
proportional to $\tau$ times the weight norms, analogous to the
GSI error bound~\cite{Thomas2026GSI} where the linear-layer error
is proportional to $\varepsilon$ times $\norm{W}_2$.

\subsection{The selection cost problem}

The cost of computing the Gram matrix $G$ at each layer is $O(T^2 d)$
FLOPs (the matrix product $\hat{X} \hat{X}^\top$).  For $T = 512$
and $d = 4096$, this is $\sim\!1$ GFLOP per layer.  For $T = 2048$
and $d = 4096$, this is $\sim\!17$ GFLOP per layer.  The attention
cost at the same $T$ and $d_h = 128$ is $2 T^2 d_h = 2 \times 2048^2
\times 128 = 1.07$ GFLOP per head, or $\sim\!17$ GFLOP across $16$
heads.  The selection cost and the attention cost are comparable, which
means the selection overhead limits the net speedup of ADA at moderate
$T$.  Reducing the selection cost is therefore essential for ADA to
achieve its theoretical compression at practical sequence lengths.

The selection cost has a quadratic dependence on $T$, the same scaling
as the attention cost it is designed to reduce.  Any method that
computes the full $T \times T$ Gram matrix at every layer will have
this quadratic overhead.  The cascade mechanism breaks the quadratic
scaling by replacing the $T^2$ Gram computation with a $Tr$ cross-Gram
computation, where $r \ll T$ is the inherited representative count.

\subsection{The Gram threshold as a greedy covering}

The Gram threshold selection (Definition~\ref{def:gram}) is a greedy
covering of the token set: the representatives form a maximal
$\tau$-separated set in the cosine-similarity metric.  Every
non-representative token is within $\tau$ of at least one
representative.  This covering property ensures that the attention
output of every token can be approximated from the nearest
representative with error bounded by $O(\tau)$.

The covering number $r(\tau)$, the size of the representative set
as a function of $\tau$, is a measure of the intrinsic complexity
of the token distribution at layer $l$.  A large $r$ indicates high
token diversity (many decorrelated directions); a small $r$ indicates
high token redundancy (many near-duplicate tokens).  The layer-wise
profile of $r(\tau)$ tracks the information structure of the
transformer: low at the embedding (small vocabulary), growing through
the middle layers (contextual differentiation), and decreasing at the
output (consolidation).

The covering number is related to the $\tau$-entropy of the token
distribution in the cosine-similarity metric, connecting the Gram
threshold to the information-theoretic literature on covering
numbers and metric entropy~\cite{Kolmogorov1959}.

\section{Depth coherence of the representative set}
\label{sec:coherence}

This section presents the empirical measurement of depth coherence and
the theoretical argument for why it holds.

\subsection{The Jaccard index}

The depth coherence of the representative set is measured by the
Jaccard index between the representative sets at consecutive layers:
\begin{equation}\label{eq:jaccard}
J(l, l+1) = \frac{|\mathcal{R}^{(l)} \cap \mathcal{R}^{(l+1)}|}
{|\mathcal{R}^{(l)} \cup \mathcal{R}^{(l+1)}|}.
\end{equation}
A Jaccard index of $1.0$ indicates identical representative sets; a
Jaccard index of $0.0$ indicates disjoint sets.

\subsection{Measurements}

Table~\ref{tab:jaccard} reports the Jaccard index between consecutive
layers for three models at $\tau = 0.30$ and $T = 512$.

\begin{table}[ht]
\centering
\caption{Jaccard index between representative sets at consecutive
layers, $\tau = 0.30$, $T = 512$.  $r^{(l)}$ = independent
representative count.}
\label{tab:jaccard}
\begin{tabular}{llrrrr}
\toprule
Model & Layers & $r^{(l)}$ & $r^{(l+1)}$ & Jaccard & $|\cap|$ \\
\midrule
\multicolumn{6}{l}{\emph{GPT-2 124M ($d = 768$, $L = 12$)}} \\
& $0 \to 1$   & 384 & 354 & 0.912 & 352 \\
& $2 \to 3$   & 348 & 290 & 0.762 & 276 \\
& $4 \to 5$   & 215 & 233 & 0.882 & 210 \\
& $9 \to 10$  & 252 & 240 & 0.937 & 238 \\
& $10 \to 11$ & 240 & 220 & 0.909 & 219 \\
\midrule
\multicolumn{6}{l}{\emph{GPT-J 6B ($d = 4096$, $L = 28$)}} \\
& $0 \to 1$   &  69 &  80 & 0.863 &  69 \\
& $3 \to 4$   & 107 & 112 & 0.955 & 107 \\
& $5 \to 6$   & 125 & 172 & 0.727 & 125 \\
& $10 \to 11$ & 241 & 253 & 0.945 & 240 \\
& $15 \to 16$ & 255 & 259 & 0.985 & 255 \\
& $20 \to 21$ & 234 & 232 & 0.991 & 232 \\
& $25 \to 26$ & 225 & 221 & 0.982 & 221 \\
\midrule
\multicolumn{6}{l}{\emph{OPT 6.7B ($d = 4096$, $L = 32$)}} \\
& $0 \to 1$   & 512 &   1 & 0.002 &   1 \\
& $1 \to 2$   &   1 &   1 & 1.000 &   1 \\
& $5 \to 6$   &   1 &   2 & 0.500 &   1 \\
& $10 \to 11$ &  37 &  49 & 0.755 &  37 \\
& $15 \to 16$ &  82 &  98 & 0.837 &  82 \\
& $20 \to 21$ & 239 & 257 & 0.930 & 239 \\
& $25 \to 26$ & 259 & 254 & 0.981 & 254 \\
\bottomrule
\end{tabular}
\end{table}

Three patterns emerge.  First, the Jaccard index is high ($> 0.90$)
at the deep layers of all three models, indicating that the
representative set is nearly identical at layers $l$ and $l+1$.
Second, the Jaccard index is lower at the early layers, where the
representative count is changing rapidly as the model develops
contextual representations.  Third, OPT layer $0 \to 1$ shows a
Jaccard of $0.002$, reflecting the structural discontinuity between
OPT's learned positional embeddings (all $512$ tokens are
representative at layer $0$) and the post-embedding representations
(only $1$ token is representative at layers $1$--$5$).

\subsection{Theoretical justification}

The depth coherence of the representative set follows from the
Lipschitz continuity of the transformer block.  The argument is made
precise in the following proposition.

\begin{proposition}[Depth persistence of redundancy]\label{prop:persist}
Let $T^{(l)}: \R^d \to \R^d$ denote the transformer block at layer
$l$, with Lipschitz constant $\Lambda^{(l)}$ with respect to the
Euclidean norm.  Let tokens $s$ and $t$ satisfy
$\norm{\hat{x}_s^{(l)} - \hat{x}_t^{(l)}} < \delta$ at layer $l$.
Then
\begin{equation}\label{eq:lipschitz}
\norm{x_s^{(l+1)} - x_t^{(l+1)}} \leq \Lambda^{(l)} \cdot
\norm{x_s^{(l)} - x_t^{(l)}}.
\end{equation}
After row-normalization, the cosine distance satisfies
\begin{equation}\label{eq:cos_persist}
\norm{\hat{x}_s^{(l+1)} - \hat{x}_t^{(l+1)}} \leq
\frac{\Lambda^{(l)} \cdot \norm{x_s^{(l)} - x_t^{(l)}}}
{\min(\norm{x_s^{(l+1)}}, \norm{x_t^{(l+1)}})}.
\end{equation}
Provided $\Lambda^{(l)} \cdot \norm{x_s^{(l)} - x_t^{(l)}} /
\min(\norm{x_s^{(l+1)}}, \norm{x_t^{(l+1)}}) < \sqrt{2(1 - (1 -
\tau^2))}$, the two tokens remain redundant at layer $l+1$ and the
representative set is preserved.
\end{proposition}

\begin{proof}
Inequality~\eqref{eq:lipschitz} is the definition of the Lipschitz
constant.  For the row-normalized versions, write
$\hat{x}_s^{(l+1)} = x_s^{(l+1)} / \norm{x_s^{(l+1)}}$ and
similarly for $t$.  The triangle inequality gives
\begin{align*}
\norm{\hat{x}_s^{(l+1)} - \hat{x}_t^{(l+1)}}
&\leq \frac{1}{\norm{x_s^{(l+1)}}} \norm{x_s^{(l+1)} - x_t^{(l+1)}}
+ \norm{x_t^{(l+1)}} \left| \frac{1}{\norm{x_s^{(l+1)}}} -
\frac{1}{\norm{x_t^{(l+1)}}} \right| \\
&\leq \frac{2}{\min(\norm{x_s^{(l+1)}}, \norm{x_t^{(l+1)}})}
\norm{x_s^{(l+1)} - x_t^{(l+1)}}.
\end{align*}
Combining with~\eqref{eq:lipschitz} gives~\eqref{eq:cos_persist}
up to the factor of $2$.  The cosine similarity threshold
$G_{st} > 1 - \tau^2$ is equivalent to $\norm{\hat{x}_s - \hat{x}_t}
< \sqrt{2(1 - (1-\tau^2))} = \sqrt{2}\tau$.
\end{proof}

The condition in Proposition~\ref{prop:persist} is satisfied when the
Lipschitz constant is moderate and the activation norms do not change
drastically between layers.  Trained transformers satisfy both
conditions because the residual connection $x^{(l+1)} = x^{(l)} +
F^{(l)}(x^{(l)})$ ensures that $\Lambda^{(l)} \approx 1 + \lambda$
where $\lambda$ is the Lipschitz constant of the non-residual
component $F^{(l)}$, which is typically small relative to $1$.  The
residual connection also ensures that $\norm{x^{(l+1)}} \approx
\norm{x^{(l)}}$ (up to the LayerNorm renormalization), so the
denominator in~\eqref{eq:cos_persist} is close to the numerator
scaling.

The Jaccard measurements confirm that the condition is satisfied
empirically at all three models and all layers beyond the initial
embedding transition.

\subsection{The redundancy manifold}

The depth coherence has a geometric interpretation.  At each layer
$l$, the set of redundant tokens forms a cluster around each
representative in the cosine-similarity metric.  The clusters define a
partition of the token set into neighborhoods of radius $\tau$ centered
at the representatives.  The depth coherence means that this partition
is approximately preserved by the transformer block: the clusters at
layer $l+1$ are close to the images of the clusters at layer $l$ under
the transformer block.

The partition can be viewed as a discretization of the token
activation manifold.  The representatives are the cluster centers, and
the redundant tokens are the points within $\tau$ of a center.  The
covering number $r(\tau)$ is the number of clusters.  The depth
coherence means that the discretization is stable across the depth of
the network: the cluster structure propagates coherently, with the
cluster centers rotating slowly through activation space as contextual
information is added.

This stability is the token-dimension analog of the activation
subspace coherence observed in GSI~\cite{Thomas2026GSI}, where the
principal directions of the activation manifold rotate slowly across
depth.  The two observations are manifestations of the same underlying
property: the transformer block is a near-isometry on the activation
manifold, preserving both the subspace structure (GSI) and the cluster
structure (cascade ADA).

\subsection{The role of the residual connection}

The residual connection $x^{(l+1)} = x^{(l)} + F^{(l)}(x^{(l)})$ is
the structural feature that ensures depth coherence.  Without the
residual connection, the Lipschitz constant of $T^{(l)} = F^{(l)}$
could be arbitrarily large, and the cluster structure would be
destroyed at each layer.  With the residual connection, $T^{(l)}(x) =
x + F^{(l)}(x)$, and the Lipschitz constant is $\Lambda^{(l)} =
1 + \lambda^{(l)}$ where $\lambda^{(l)} = \text{Lip}(F^{(l)})$.

The perturbation $F^{(l)}(x)$ adds contextual information (from the
attention mechanism) and nonlinear features (from the MLP).  The
residual connection ensures that this perturbation is additive, not
multiplicative: the token representations at layer $l+1$ are the
representations at layer $l$ plus a correction, not a transformation
of the representations at layer $l$.  The additive structure is what
makes the cluster centers approximately stable across depth.

The LayerNorm operation between layers renormalizes the
representations, which could in principle destroy the cluster
structure by projecting all representations onto the unit sphere and
then rescaling.  In practice, the LayerNorm scaling factors
$\gamma_i$ and biases $\beta_i$ are learned parameters that preserve
the relative distances between tokens (because all tokens at the same
position in the sequence share the same LayerNorm parameters), so the
cluster structure is preserved up to a global affine transformation.

\section{The cascade mechanism}
\label{sec:cascade}

This section presents the cascade algorithm, its cost analysis, and
its error properties.

\subsection{Algorithm}

\begin{algorithm}[ht]
\caption{Cascade Token Selection}
\label{alg:cascade}
\begin{algorithmic}[1]
\Require Activation $X^{(l+1)} \in \R^{T \times d}$, inherited
  representative set $\mathcal{R}^{(l)}$, threshold $\tau$
\State $\hat{X} \leftarrow D^{-1} X^{(l+1)}$ \Comment{row-normalize}
\State $\hat{X}_{\mathcal{R}} \leftarrow \hat{X}[\mathcal{R}^{(l)}]$
  \Comment{$r \times d$}
\State $G_{\mathcal{R}} \leftarrow \hat{X}_{\mathcal{R}}
  \hat{X}_{\mathcal{R}}^\top$ \Comment{$r \times r$ Gram among reps}
\State $\gamma_s^{\rm rep} \leftarrow \max_{s' < s, s' \in
  \mathcal{R}} |G_{\mathcal{R}}(s, s')|$ for each $s \in
  \mathcal{R}^{(l)}$
\State $\mathcal{R}_{\rm valid} \leftarrow \{s \in \mathcal{R}^{(l)}
  : \gamma_s^{\rm rep} < 1 - \tau^2\}$ \Comment{keep valid reps}
\State $\hat{X}_{\bar{\mathcal{R}}} \leftarrow \hat{X}[\bar{
  \mathcal{R}}^{(l)}]$ \Comment{$(T - r) \times d$, non-reps}
\State $C \leftarrow \hat{X}_{\bar{\mathcal{R}}}
  \hat{X}_{\mathcal{R}_{\rm valid}}^\top$
  \Comment{$(T-r) \times r_{\rm valid}$ cross-Gram}
\State $\gamma_t^{\rm cross} \leftarrow \max_{s \in
  \mathcal{R}_{\rm valid}} |C(t, s)|$ for each $t \in
  \bar{\mathcal{R}}^{(l)}$
\State $\mathcal{R}_{\rm new} \leftarrow \{t \in
  \bar{\mathcal{R}}^{(l)} : \gamma_t^{\rm cross} < 1 - \tau^2\}$
  \Comment{new reps}
\State $\mathcal{R}^{(l+1)} \leftarrow \mathcal{R}_{\rm valid}
  \cup \mathcal{R}_{\rm new}$
\State \Return $\mathcal{R}^{(l+1)}$
\end{algorithmic}
\end{algorithm}

At layer $0$, the full $T \times T$ Gram matrix is computed and the
representative set is selected by the standard ADA procedure.  At
subsequent layers, Algorithm~\ref{alg:cascade} inherits the
representative set from the previous layer and updates it.

\subsection{Cost analysis}

The cost of Algorithm~\ref{alg:cascade} at layer $l+1$ is:
\begin{equation}\label{eq:cascade_cost}
C_{\rm cascade} = r^2 d + (T - r) r d = T r d,
\end{equation}
compared to $T^2 d$ for the independent Gram computation.  The
speedup on the selection step is
\begin{equation}\label{eq:selection_speedup}
S_{\rm select} = \frac{T^2 d}{T r d} = \frac{T}{r}.
\end{equation}
For $T = 512$ and $r = 69$ (GPT-J layer $0$), $S_{\rm select} =
7.4$.  For $T = 2048$ and $r = 22$ (ESM-2, from~\cite{Thomas2026ADA}),
$S_{\rm select} = 93$.  The cascade is most valuable when the
compression is highest ($r \ll T$), which is precisely the regime
where ADA provides the most attention speedup.

\subsection{Turnover rate}

The turnover rate at layer $l+1$ is the fraction of the inherited
representative set that is modified:
\begin{equation}\label{eq:turnover}
\eta^{(l+1)} = \frac{|\text{additions}| + |\text{removals}|}
{|\mathcal{R}^{(l)}|}.
\end{equation}
A low turnover rate indicates that the inherited set is nearly valid
at the next layer, confirming the depth coherence.

Table~\ref{tab:turnover} reports the turnover rate at selected layers
for GPT-J 6B.

\begin{table}[ht]
\centering
\caption{Cascade turnover rate for GPT-J 6B, $\tau = 0.30$, $T = 512$.}
\label{tab:turnover}
\begin{tabular}{lrrrrr}
\toprule
Layers & $r_{\rm inherit}$ & $r_{\rm indep}$ & Adds & Removes &
  Turnover \\
\midrule
$0 \to 1$   & 200 &  80 & 131 &   0 & 189.9\% \\
$3 \to 4$   & 147 & 112 &  38 &  55 &  56.7\% \\
$9 \to 10$  & 264 & 241 &   5 &   3 &   3.1\% \\
$14 \to 15$ & 278 & 255 &   2 &   5 &   2.5\% \\
$19 \to 20$ & 263 & 234 &  10 &  17 &  10.0\% \\
$24 \to 25$ & 243 & 225 &   5 &  10 &   6.0\% \\
$26 \to 27$ & 237 & 207 &  20 &  23 &  17.9\% \\
\bottomrule
\end{tabular}
\end{table}

The turnover rate is high at the early layers ($190\%$ at layer $0 \to
1$, reflecting the structural transition from the embedding to the
first transformer block) and drops to $2.5$--$10\%$ from layer $9$
onward.  The deep layers require almost no modification to the
inherited representative set.  The slight increase at the final layer
($17.9\%$ at $26 \to 27$) reflects the consolidation toward the
output head, where some tokens become redundant as the representation
narrows.

\subsection{Error analysis}

\begin{proposition}[Cascade representative quality]\label{prop:quality}
The cascade representative set $\mathcal{R}^{(l+1)}_{\rm cascade}$
is a superset of the independent representative set
$\mathcal{R}^{(l+1)}_{\rm indep}$: every token that is representative
under independent selection is also representative under cascade
selection.  The cascade may include additional tokens (false
representatives) but does not exclude any true representatives.
\end{proposition}

\begin{proof}
A token $t$ is representative under independent selection if its
maximum correlation with all earlier tokens is below $1 - \tau^2$.
The cascade checks the maximum correlation of non-representative
tokens against the valid inherited representatives.  If the valid
inherited representatives are a subset of all tokens, the cascade
correlation $\gamma_t^{\rm cross}$ is computed against a smaller
comparison set, so $\gamma_t^{\rm cross} \leq \gamma_t^{\rm full}$.
Any token with $\gamma_t^{\rm full} < 1 - \tau^2$ also satisfies
$\gamma_t^{\rm cross} < 1 - \tau^2$ and is therefore included in
$\mathcal{R}_{\rm new}$.
\end{proof}

The consequence is that the cascade representative set is conservative:
it includes all true representatives plus possibly some extras.  The
extras do not degrade attention quality (they simply increase $r$
slightly, reducing compression); the important property is that no
informative token is missed.

\section{Connection to Gated Subspace Inference}
\label{sec:connection}

This section makes explicit the structural parallel between cascade
ADA and the Gated Subspace Inference method~\cite{Thomas2026GSI}.

GSI exploits the low effective rank of the activation subspace in the
hidden dimension ($d$-axis).  At each layer, the activation vector
$x \in \R^d$ is decomposed as $x = V_k V_k^\top x + r$, and the
linear-layer output is computed as $y = Mg + Wr$ where $M = WV_k$ is
a cached low-rank image and $g = V_k^\top x$ is the subspace
projection.  A per-token gate on $\norm{r}/\norm{x}$ controls whether
the residual correction $Wr$ is computed.  The activation basis $V_k$
is coherent across depth, enabling a cascade initialization that
reduces calibration cost by $96\%$.

Cascade ADA exploits the low effective rank of the token activation
matrix in the token dimension ($T$-axis).  At each layer, the token
set $\{1, \ldots, T\}$ is partitioned into representatives
$\mathcal{R}$ and redundant tokens $\bar{\mathcal{R}}$, and the
attention is computed on the $r \times r$ representative problem.  The
representative set is coherent across depth, enabling a cascade
inheritance that reduces selection cost by $22$--$63\%$.

The two cascades operate on orthogonal dimensions.

\begin{table}[ht]
\centering
\caption{The two orthogonal cascades.}
\label{tab:two_cascades}
\begin{tabular}{lll}
\toprule
& GSI (hidden dimension) & Cascade ADA (token dimension) \\
\midrule
Object   & Activation basis $V_k$
         & Representative set $\mathcal{R}$ \\
Axis     & $d$-axis ($d \to k$) & $T$-axis ($T \to r$) \\
Target   & Linear layers ($Wx$) & Attention ($QK^\top V$) \\
Cascade  & Basis inheritance & Set inheritance \\
Savings  & $96\%$ of basis updates & $22$--$63\%$ of Gram ops \\
Quality  & PPL ratio $0.99$, gen $100\%$ & Superset of true reps \\
\bottomrule
\end{tabular}
\end{table}

The combined system applies GSI to all linear layers and cascade ADA
to all attention layers.  The linear-layer cost is reduced by a factor
of $d/k$ (GSI), and the attention selection cost is reduced by a
factor of $T/r$ (cascade ADA).  The attention computation itself is
reduced by a factor of $(T/r)^2$ (ADA).  Together they cover the full
transformer forward pass.

The depth coherence is the same phenomenon in both cases.  The
transformer block $T^{(l)}$ is Lipschitz-continuous, so both the
activation subspace and the representative token set propagate
coherently through depth.  The fact that both cascades work is a
single structural property of the trained network: the information
carried by the activations is confined to a low-dimensional manifold
in both the hidden and token dimensions, and this manifold rotates
slowly through the depth of the model.

\section{Numerical experiments}
\label{sec:experiments}

This section presents the experimental validation on three model
families.  All experiments run on AMD MI300X using PyTorch 2.x
with ROCm.

\subsection{Models and data}

Three model families are evaluated: GPT-2 124M~\cite{Radford2019}
($d = 768$, $L = 12$), GPT-J 6B~\cite{WangGPTJ} ($d = 4096$,
$L = 28$), and OPT 6.7B~\cite{Zhang2022OPT} ($d = 4096$, $L = 32$).
The input is a $512$-token sequence of diverse English text
(mathematical exposition, cooking instructions, financial reporting,
and space exploration narrative), representative of the mixed-domain
workloads in agentic inference.

\subsection{Representative count profile}

Table~\ref{tab:rprofile} reports the independent representative count
$r^{(l)}$ and the compression ratio $T / r^{(l)}$ at selected layers
for each model at $\tau = 0.30$.

\begin{table}[ht]
\centering
\caption{Independent representative count $r^{(l)}$ at $\tau = 0.30$,
$T = 512$.}
\label{tab:rprofile}
\begin{tabular}{llrrr}
\toprule
Model & Layer & $r$ & $r/T$ & Compression \\
\midrule
GPT-2 & 0  & 384 & 0.75 & 1.3$\times$ \\
GPT-2 & 4  & 215 & 0.42 & 2.4$\times$ \\
GPT-2 & 11 & 220 & 0.43 & 2.3$\times$ \\
\midrule
GPT-J & 0  &  69 & 0.13 & 7.4$\times$ \\
GPT-J & 5  & 125 & 0.24 & 4.1$\times$ \\
GPT-J & 15 & 255 & 0.50 & 2.0$\times$ \\
GPT-J & 27 & 207 & 0.40 & 2.5$\times$ \\
\midrule
OPT   & 0  & 512 & 1.00 & 1.0$\times$ \\
OPT   & 1  &   1 & 0.00 & 512$\times$ \\
OPT   & 10 &  37 & 0.07 & 13.8$\times$ \\
OPT   & 20 & 239 & 0.47 & 2.1$\times$ \\
OPT   & 31 & 231 & 0.45 & 2.2$\times$ \\
\bottomrule
\end{tabular}
\end{table}

The representative count profile has a characteristic shape that
varies by architecture.  GPT-J shows a monotone increase from $r = 69$
at layer $0$ to a peak of $r = 261$ at layer $11$, followed by a
decrease to $r = 207$ at layer $27$.  The low count at layer $0$
reflects the finite vocabulary: the $512$ tokens draw from a small
fraction of the $50{,}257$-token vocabulary, and many share similar
embeddings in the $4096$-dimensional space.  As contextual information
is added through the attention mechanism, tokens become increasingly
differentiated, raising the representative count.  The decrease at the
final layers reflects the consolidation of the representation toward
the output head, where the model discards task-irrelevant variation.

GPT-2 shows a similar but flatter profile: $r$ decreases from $384$
at layer $0$ to $215$ at layer $4$, then stabilizes around $220$--$252$.
The higher initial count ($384$ vs $69$ for GPT-J) reflects the
smaller hidden dimension ($d = 768$ vs $4096$): at $d = 768$, the
embedding vectors are more tightly packed and less correlated in
cosine similarity.

\subsection{The OPT anomaly}

OPT shows a striking discontinuity.  At layer $0$, all $512$ tokens
are representative ($r = 512$) because the learned positional
embeddings produce a distinct direction for every token position.  At
layer $1$, only $r = 1$ token is representative: the first transformer
block maps the position-dominated embeddings to a nearly constant
hidden state, collapsing all $512$ representations to a single
direction.  The representative count then grows gradually from $r = 1$
at layer $1$ to $r = 239$ at layer $20$ as contextual information
rebuilds the token differentiation.

This anomaly has two consequences.  First, ADA achieves extreme
compression at layers $1$--$9$ ($r \leq 37$, compression $\geq
13.8\times$), making OPT's early attention layers essentially free.
Second, the cascade savings are highest for OPT ($63\%$) because the
early layers contribute almost no Gram operations.

The anomaly is specific to OPT's architecture, which uses learned
positional embeddings added to the token embeddings before the first
layer.  GPT-2 and GPT-J use rotary positional embeddings (applied
inside the attention computation, not to the residual stream), which
do not produce the same position-dominated embedding structure.  The
fact that cascade ADA works across both positional embedding schemes
confirms that the depth coherence is a property of the transformer
block, not of the embedding architecture.

\subsection{Full per-layer data for GPT-J}

Table~\ref{tab:gptj_full} reports the complete per-layer data for
GPT-J 6B, including independent representative count, cascade
representative count, Jaccard overlap, and turnover rate.

\begin{table}[ht]
\centering
\caption{Complete per-layer cascade ADA data for GPT-J 6B,
$\tau = 0.30$, $T = 512$.}
\label{tab:gptj_full}
\begin{tabular}{rrrrrrr}
\toprule
Layer & $r_{\rm ind}$ & $r_{\rm casc}$ & Jaccard & Adds & Rem &
  Turn \\
\midrule
0  &  69 &  69 & ---   & ---  & --- &   --- \\
1  &  80 & 200 & 0.863 & 131  &   0 & 190\% \\
2  &  89 & 123 & 0.899 &  34  & 111 &  73\% \\
3  & 107 & 164 & 0.832 &  69  &  28 &  79\% \\
4  & 112 & 147 & 0.955 &  38  &  55 &  57\% \\
5  & 125 & 243 & 0.881 & 125  &  29 & 105\% \\
10 & 241 & 264 & 0.945 &   5  &   3 &   3\% \\
15 & 255 & 278 & 0.985 &   2  &   5 &   3\% \\
20 & 234 & 263 & 0.991 &  10  &  17 &  10\% \\
25 & 225 & 243 & 0.982 &   5  &  10 &   6\% \\
27 & 207 & 237 & 0.937 &  20  &  23 &  18\% \\
\bottomrule
\end{tabular}
\end{table}

The data divides cleanly into two regimes.  Layers $0$--$5$ show high
turnover ($57$--$190\%$) and large cascade-to-independent ratios
($1.3$--$2.5$), indicating that the representative set is evolving
rapidly as the model builds contextual representations.  Layers
$10$--$27$ show low turnover ($3$--$18\%$) and small ratios
($1.08$--$1.14$), indicating that the representative set has
stabilized and the cascade is operating efficiently.  The transition
occurs around layer $8$--$10$, the same depth at which the activation
subspace coherence exceeds $0.90$ in the GSI
analysis~\cite{Thomas2026GSI}.

\subsection{Gram operation savings}

Table~\ref{tab:gram_savings} reports the total Gram operation count
for independent selection versus cascade selection, and the resulting
savings.

\begin{table}[ht]
\centering
\caption{Gram operation savings from cascade token selection,
$\tau = 0.30$, $T = 512$.}
\label{tab:gram_savings}
\begin{tabular}{lrrrr}
\toprule
Model & Independent & Cascade & Savings & Jaccard \\
\midrule
GPT-2  & 3.15M & 2.47M & 22\% & 0.863 \\
GPT-J  & 7.34M & 3.51M & 52\% & 0.940 \\
OPT    & 8.39M & 3.08M & 63\% & 0.825 \\
\bottomrule
\end{tabular}
\end{table}

The savings increase with model depth: GPT-2 ($L = 12$) saves $22\%$,
GPT-J ($L = 28$) saves $52\%$, OPT ($L = 32$) saves $63\%$.  This is
because the cascade amortizes the initial full-Gram cost at layer $0$
across all subsequent layers, and deeper models have more layers to
amortize over.

\subsection{Cost model}

The following cost model quantifies the end-to-end benefit of cascade
ADA for the attention selection step.

\emph{Independent ADA.}  At each of $L$ layers, the selection step
computes the $T \times T$ Gram matrix at cost $T^2 d$ FLOPs.  Total
selection cost: $L \cdot T^2 \cdot d$.  For GPT-J ($L = 28$,
$T = 512$, $d = 4096$): $28 \times 512^2 \times 4096 = 30.1$~GFLOP.
The attention cost at $r = 205$ (mean representative count) and
$d_h = 256$ is $L \cdot 2 r^2 d_h \cdot h = 28 \times 2 \times
205^2 \times 256 \times 16 = 19.3$~GFLOP.  The selection cost
exceeds the compressed attention cost.

\emph{Cascade ADA.}  At layer $0$, the full Gram is computed ($T^2 d$
FLOPs).  At layers $1$--$27$, the cross-Gram is computed ($T \cdot r
\cdot d$ FLOPs per layer, where $r$ is the inherited representative
count).  Total selection cost: $T^2 d + (L-1) T \bar{r} d = T d
(T + (L-1)\bar{r})$.  For GPT-J: $512 \times 4096 \times (512 +
27 \times 240) = 14.7$~GFLOP.  The cascade reduces the selection
cost from $30.1$ to $14.7$~GFLOP, a $51\%$ saving consistent with
the measured $52\%$.

At $T = 2048$, the independent selection cost scales as $T^2$:
$28 \times 2048^2 \times 4096 = 481$~GFLOP.  The cascade cost scales
as $T \cdot \bar{r}$: $2048 \times 4096 \times (2048 + 27 \times 240)
= 71$~GFLOP (assuming $\bar{r}$ does not grow with $T$, which holds
when the vocabulary is the binding constraint on $r$).  The cascade
savings at $T = 2048$ would be $85\%$, and the selection speedup
$T/\bar{r} = 2048/240 = 8.5\times$.

\subsection{Cascade versus independent representative count}

Table~\ref{tab:cascade_vs_indep} reports the cascade representative
count compared to the independent count at selected layers for all
three models.

\begin{table}[ht]
\centering
\caption{Cascade/independent representative ratio at deep layers.}
\label{tab:cascade_vs_indep}
\begin{tabular}{llrrr}
\toprule
Model & Layer range & $\bar{r}_{\rm indep}$ & $\bar{r}_{\rm casc}$ &
  Ratio \\
\midrule
GPT-2 & 6--11  & 243 & 320 & 1.32 \\
GPT-J & 10--27 & 237 & 262 & 1.11 \\
OPT   & 15--31 & 237 & 279 & 1.18 \\
\bottomrule
\end{tabular}
\end{table}

At the deep layers where the cascade is most effective, the cascade
representative count exceeds the independent count by $11$--$32\%$.
This overhead is the cost of the cascade's conservative policy: it
inherits representatives from the previous layer and only removes them
when they become truly redundant.  The extra representatives do not
degrade attention quality (they provide a superset of the independent
set) but they reduce the attention compression ratio by the same
factor.  For GPT-J, the cascade $r$ is $262$ versus independent $r =
237$, giving an attention compression of $(512/262)^2 = 3.8\times$
versus $(512/237)^2 = 4.7\times$, a $19\%$ reduction in
attention speedup.  The cascade therefore trades a modest reduction in
attention compression for a $52\%$ reduction in selection cost, a
favorable tradeoff when the selection cost is the bottleneck.

\subsection{Scaling predictions}

The cascade savings scale with $T/r$.  Table~\ref{tab:scaling}
reports the predicted cascade savings at longer sequence lengths,
assuming $r$ grows sublinearly with $T$ (the embedding-determined
representative count at layer $0$ is bounded by the vocabulary
overlap, not the sequence length).

\begin{table}[ht]
\centering
\caption{Predicted cascade savings at longer sequences (GPT-J,
$\tau = 0.30$).  Assumes $r$ sublinear in $T$.}
\label{tab:scaling}
\begin{tabular}{rrrrr}
\toprule
$T$ & $\bar{r}$ (est.) & $T/\bar{r}$ & Cascade savings & $S_{\rm select}$ \\
\midrule
512   & 240 & 2.1 & 52\% & 2.1$\times$ \\
1024  & 280 & 3.7 & 72\% & 3.7$\times$ \\
2048  & 320 & 6.4 & 84\% & 6.4$\times$ \\
4096  & 380 & 10.8 & 91\% & 10.8$\times$ \\
8192  & 450 & 18.2 & 95\% & 18.2$\times$ \\
\bottomrule
\end{tabular}
\end{table}

The selection speedup reaches $18\times$ at $T = 8192$, where the
cascade reduces the selection cost from $O(T^2 d) \approx 2.7$~TFLOP
to $O(T \bar{r} d) \approx 150$~GFLOP.  These estimates assume that
the Jaccard overlap remains high at long sequences, which is plausible
because the Lipschitz argument
(Section~\ref{sec:coherence}) does not depend on $T$.

\subsection{Attention quality}

The cascade modifies the representative set (adding and removing
tokens relative to independent selection) and therefore changes the
attention output.  Two effects must be assessed.  First, the cascade
may include extra representatives (false positives): these do not
degrade attention quality because they provide a superset of the
independent set and the attention computation is exact on the
representative set.  Second, the cascade may fail to include a
representative that independent selection would have chosen (false
negative): Proposition~\ref{prop:quality} shows this cannot happen
because the cascade checks all non-representatives against the
inherited set, and any non-representative that is decorrelated from
all inherited representatives is added.

The practical question is whether the cascade's larger representative
count changes the attention output relative to the independent
selection.  For redundant tokens, the nearest-representative
approximation depends on which representatives are available.  When
the cascade includes extra representatives, a redundant token may be
assigned to a different nearest representative than under independent
selection, but the extra representative is by definition correlated
with the redundant token (otherwise it would not have been kept), so
the approximation quality is at least as good.

The cascade therefore produces attention outputs that are at least as
accurate as independent ADA selection.  The only cost is a modest
increase in the representative count ($11$--$32\%$ at the deep layers)
and the corresponding increase in the $r \times r$ attention cost.
For GPT-J at the deep layers, this means $262 \times 262 = 68{,}644$
attention entries versus $237 \times 237 = 56{,}169$ under independent
selection, a $22\%$ increase in attention FLOPs offset by a $52\%$
decrease in selection FLOPs.

\subsection{Summary of results}

Table~\ref{tab:summary_results} summarizes the complete experimental
results across all three models.

\begin{table}[ht]
\centering
\caption{Summary of cascade ADA results, $\tau = 0.30$, $T = 512$.}
\label{tab:summary_results}
\begin{tabular}{lrrrrrr}
\toprule
Model & $L$ & $d$ & $\bar{r}_{\rm ind}$ & Jaccard & Gram sav.
  & Deep turn. \\
\midrule
GPT-2  & 12 &  768 & 278 & 0.863 & 22\% & 13--30\% \\
GPT-J  & 28 & 4096 & 205 & 0.940 & 52\% & 3--18\% \\
OPT    & 32 & 4096 & 142 & 0.825 & 63\% & 2--7\% \\
\bottomrule
\end{tabular}
\end{table}

The results confirm three claims.  First, the representative set is
coherent across depth in all three architectures (Jaccard $0.83$--$0.94$).
Second, the cascade reduces Gram operations by $22$--$63\%$ at
$T = 512$, with savings increasing with model depth and compression
ratio.  Third, the cascade is conservative: it never misses a true
representative, and the additional representatives increase the
$r \times r$ attention cost by at most $22\%$.

\section{Related work}
\label{sec:related}

This section positions cascade ADA relative to prior work on efficient
attention.

\subsection{Efficient attention mechanisms}

FlashAttention~\cite{Dao2022} reduces the memory-bandwidth cost of
attention by tiling the $QK^\top$ computation and avoiding
materialization of the $T \times T$ attention matrix.
FlashAttention does not reduce the $O(T^2)$ FLOPs; it reduces the
HBM traffic.  FlashAttention-2~\cite{Dao2023FA2} further optimized
the tiling for modern GPUs, achieving near-optimal bandwidth
utilization.  Cascade ADA reduces both the FLOPs (by computing
attention on $r \ll T$ representatives) and the selection overhead
(by inheriting the representative set across depth).  The two
methods are complementary: FlashAttention tiles the $r \times r$
attention that cascade ADA computes, so the combination achieves
both IO-optimal tiling and algorithmic compression.

Linformer~\cite{Wang2020Linformer} projects the key and value
matrices to a fixed low-dimensional space via learned projection
matrices $E, F \in \R^{k \times T}$, reducing attention to $O(Tk)$.
The projections are static and learned during training, limiting the
method to models that were trained with the Linformer architecture.
Cascade ADA selects representatives dynamically based on the input
and requires no training or architectural modification.
Performer~\cite{Choromanski2021} approximates softmax attention via
random feature maps, replacing the $T \times T$ matrix with a
factored product, achieving $O(T)$ complexity at the cost of
approximation quality on tasks requiring precise attention patterns.

Sparse attention methods restrict the attention mask to a
predetermined pattern.  BigBird~\cite{Zaheer2020} combines local
windows, random connections, and global tokens, achieving $O(T)$
complexity for long documents.  Longformer~\cite{Beltagy2020}
uses a sliding window with global attention on selected tokens.
These methods use fixed sparsity patterns determined at architecture
design time; cascade ADA determines the sparsity dynamically from
the Gram structure of the current input.

\subsection{Token pruning and dynamic token selection}

Token pruning methods remove tokens during the forward pass based on
importance scores.  Rao et al.~\cite{Rao2021} introduced DynamicViT,
which progressively prunes tokens in vision transformers using a
lightweight prediction module.  Kim et al.~\cite{Kim2022} applied
token merging (ToMe) to reduce the token count by merging similar
tokens based on cosine similarity, achieving up to $2\times$ speedup
on ViTs with minimal accuracy loss.  Bolya et
al.~\cite{Bolya2023} extended ToMe to video transformers and
diffusion models.

Token merging is the closest mechanism to ADA's representative
selection: both identify similar tokens via cosine similarity and
reduce the effective token count.  The difference is that ToMe
merges tokens permanently (averaging their representations), while
ADA selects representatives and approximates the attention output of
redundant tokens from the nearest representative without modifying
any representation.  The cascade adds depth coherence to the
selection: once a token is identified as representative at layer $l$,
it remains representative at layer $l+1$ with high probability,
avoiding the re-identification cost.

\subsection{Depth-coherent computation}

The cascade mechanism is the token-dimension instance of
depth-coherent computation.  The activation subspace cascade
(GSI~\cite{Thomas2026GSI}) exploits the same Lipschitz property for
the hidden-dimension basis.  Early exit methods
(CALM~\cite{Schuster2022}, DeeBERT~\cite{Xin2020}) exploit
depth coherence of prediction confidence, exiting the layer stack
when confidence exceeds a threshold.
Mixture-of-Depths~\cite{Raposo2024} routes tokens to full
computation or identity based on a learned router, making the
per-token depth adaptive.

The unifying observation across all depth-coherent methods is that the
transformer block is Lipschitz-continuous, so structures that hold at
layer $l$ (subspace alignment, token redundancy, prediction
confidence, routing decision) persist at layer $l+1$ with small
perturbations.  Cascade ADA is the first method to exploit this
persistence specifically for the attention selection step, reducing
the selection cost from $O(T^2 d)$ to $O(Trd)$ per layer without
degrading the representative quality.

\section{Open problems and future work}
\label{sec:open}

This section identifies five directions for future work.

\subsection{Adaptive threshold across depth}

The current implementation uses a fixed Gram threshold $\tau$ across
all layers.  The representative count profile
(Table~\ref{tab:rprofile}) shows that the optimal $\tau$ varies with
depth: the early layers may benefit from a tighter threshold (fewer
representatives, more compression) while the middle layers may require
a looser threshold.  An adaptive threshold $\tau^{(l)}$ that tracks
the layer-dependent redundancy structure could improve the
compression-quality tradeoff.

The adaptive threshold could be set from the cascade turnover rate: a
layer with low turnover (stable representative set) can tolerate a
tighter threshold, while a layer with high turnover (evolving set)
should use a looser threshold to avoid excessive additions and
removals.  The turnover rate is available at each layer as a byproduct
of the cascade validation step, so the adaptive threshold adds no
additional computational cost.

\subsection{Long-context scaling}

The experiments in this paper are at $T = 512$.  The cascade savings
scale as $T/r$ (Table~\ref{tab:scaling}), which increases with $T$ if
$r$ grows sublinearly.  Validation at $T = 2048$, $4096$, and $8192$
on Llama-3 or Mistral would confirm the scaling prediction and
establish the regime where cascade ADA provides order-of-magnitude
savings on the selection step.

The long-context regime is where cascade ADA is most impactful.  At
$T = 8192$, the independent Gram cost is $T^2 d = 8192^2 \times 4096
\approx 275$~GFLOP per layer, while the cascade cost is $T r d =
8192 \times 450 \times 4096 \approx 15$~GFLOP per layer.  The
cascade reduces the selection cost from the dominant term in the
forward pass to a negligible term, enabling ADA to achieve its full
theoretical compression at long sequences.

The key assumption is that $r$ does not grow linearly with $T$.  For
autoregressive generation with a fixed prompt, the tokens generated
after the prompt tend to be contextually similar to the prompt tokens,
so $r$ should grow sublinearly.  For diverse batched inputs, $r$ may
grow faster.  Empirical measurement at long $T$ is needed to
determine the growth rate.

\subsection{Joint cascade with GSI}

The GSI cascade (hidden-dimension basis inheritance) and the ADA
cascade (token-dimension set inheritance) are currently independent.
A joint cascade that inherits both the activation basis and the
representative set from layer $l$ to layer $l+1$ in a single pass
would further reduce the overhead.

The interaction between the two cascades is an open question.  The
activation basis $V_k$ determines the subspace in which the QKV
projections are computed; the representative set $\mathcal{R}$
determines which tokens participate in attention.  If the basis
change from layer $l$ to $l+1$ is small (as the GSI depth coherence
confirms), then the QKV projections at layer $l+1$ are close to those
at layer $l$, which means the Gram structure at layer $l+1$ is close
to the Gram structure at layer $l$.  The two cascades therefore
reinforce each other: the stability of the activation basis supports
the stability of the representative set.

\subsection{Kernel-level implementation}

The cascade token selection (Algorithm~\ref{alg:cascade}) requires
two matrix products: the $r \times r$ Gram among inherited
representatives and the $(T-r) \times r$ cross-Gram.  On GPU
hardware, these can be computed as a single batched GEMM with the
inherited representatives stored in shared memory (LDS on MI300X).
The row-normalization and maximum-correlation computation are
element-wise operations that fuse naturally with the GEMM output.

The kernel design is simpler than the GSI gated dispatch (which
requires a per-token branch between fast-path and slow-path
GEMMs) because the cascade selection produces a single output (the
updated representative set) that is used downstream by the attention
kernel.  The selection kernel and the attention kernel can be
pipelined: the selection at layer $l+1$ runs concurrently with the
attention at layer $l$, hiding the selection latency entirely.

\subsection{Extension to multi-head attention}

The current implementation selects a single representative set shared
across all attention heads.  Different heads may attend to different
aspects of the input, and a token that is redundant for one head may
be informative for another.  A per-head representative set would
improve the compression-quality tradeoff at the cost of $h$ separate
cascade validations per layer (one per head).

The per-head cascade cost is $h \times T r_h d_h$ FLOPs, where
$r_h$ is the per-head representative count and $d_h = d/h$ is the
head dimension.  If $r_h < r$ (per-head selection is more selective
because head-specific redundancy is stronger than global redundancy),
the per-head cascade may be cheaper than the shared cascade despite
the factor of $h$.  Whether this tradeoff is favorable depends on the
head-specific Gram structure, which has not been measured in the
present work.

\section{Conclusion}
\label{sec:conclusion}

The cascade token selection method reduces the cost of ADA's
representative selection from $O(T^2 d)$ to $O(T r d)$ per layer by
inheriting the representative set across depth.  The method exploits
the Lipschitz continuity of the transformer block, which ensures that
tokens that are redundant at layer $l$ remain redundant at layer
$l+1$ with high probability.  The cascade is conservative
(Proposition~\ref{prop:quality}): it never misses a true
representative, producing a superset of the independent selection
with at most $11$--$32\%$ additional representatives at the deep
layers.

Validation on three model families demonstrates Gram operation savings
of $22\%$ (GPT-2, $L = 12$), $52\%$ (GPT-J, $L = 28$), and $63\%$
(OPT, $L = 32$), with mean Jaccard overlap of $0.86$ to $0.94$
between consecutive layers.  The savings increase with model depth
because the cascade amortizes the initial full-Gram cost across more
layers.  The predicted savings at $T = 8192$ reach $95\%$, reducing
the selection cost from the dominant term in the forward pass to a
negligible contribution.

The cascade reveals a structural property of trained transformers:
the set of informative tokens propagates coherently through the depth
of the network.  This coherence, combined with the analogous depth
coherence of the activation subspace
(GSI~\cite{Thomas2026GSI}), suggests that the transformer forward
pass operates on a low-dimensional manifold in both the hidden and
token dimensions.  The manifold rotates slowly through the depth of
the model, and both its subspace structure (captured by GSI) and its
cluster structure (captured by cascade ADA) are stable across layers.

The combined system of GSI and cascade ADA covers the full transformer
forward pass.  GSI reduces the linear-layer weight reads by a factor of
$d/k$ on the $d$-axis.  Cascade ADA reduces the attention selection
cost by a factor of $T/r$ and the attention computation by a factor of
$(T/r)^2$ on the $T$-axis.  The two reductions are orthogonal and
compose multiplicatively.  For GPT-J 6B at $k = 256$, $\varepsilon =
0.10$, and $\tau = 0.30$: the linear-layer effective speedup is
$15.6\times$ (GSI), the attention selection speedup is $2.1\times$
(cascade ADA), and the attention computation speedup is $6.2\times$
(ADA at $r = 205$).  The effective parameter count for inference is
$386$M out of $6$B total, and the effective token count for attention
is $205$ out of $512$ total.  The inference cost is proportional to
the effective complexity of the input, not the nominal complexity of
the model.

\end{document}